%% file: camera_ready_main.tex
\documentclass{article}
\usepackage{ijcai18}
\usepackage{times}
\usepackage{xcolor}
\usepackage{soul}
\usepackage[utf8]{inputenc}
\usepackage[small]{caption}

\usepackage{url}
\usepackage{amsmath,amssymb,amsfonts,amsthm}
\usepackage{dsfont}
\usepackage{algorithm}
\usepackage{algorithmicx}
\usepackage{algpseudocode}
\usepackage{bm}
\usepackage{subfigure}
\usepackage{graphicx}
\graphicspath{{../figures/}}

\DeclareMathOperator*{\argmax}{arg\,\max}

\DeclareMathOperator*{\Ex}{\mathbb{E}}

\newcommand{\adclick}{\emph{Ad-Clicks}}
\newcommand{\coupon}{\emph{Coupon-Purchase}}
\newcommand{\edxcourse}{\emph{edX-Course}}

\newcommand{\myucb}{\textbf{Con-UCB}}
\newcommand{\lexp}{\textsc{LExp}}
\newcommand{\ucb}{\textsc{CUCB}}
\newcommand{\expm}{\textsc{Exp3.M}}

\newcommand{\algotirhmicinitialize}{\textbf{Initialize:}} 
\newcommand{\Initial}{\item[\algotirhmicinitialize]}

\definecolor{dblue}{rgb}{0.2, 0.2, 0.9}
\newcommand{\blue}[1]{\textcolor{black}{#1}}
\usepackage[normalem]{ulem}

\newtheorem{lemma}{\textbf{Lemma}}

\newtheorem{theorem}{\textbf{Theorem}}

\allowdisplaybreaks

\begin{document}

\title{Beyond the Click-Through Rate: Web Link Selection with
  Multi-level Feedback}

  \author{Kun Chen\textsuperscript{*},
  Kechao Cai\textsuperscript{$\dag$},
  Longbo Huang\textsuperscript{$\ddag$},
  John C.S. Lui\textsuperscript{$\dag$}, \\
  \textsuperscript{*}\textsuperscript{$\ddag$}{Institute for Interdisciplinary Information Sciences (IIIS), Tsinghua University}\\
  \textsuperscript{$\dag$}{Department of Computer Science \& Engineering,  The Chinese University of Hong Kong}\\
  \textsuperscript{*}chenkun14@mails.tsinghua.edu.cn,
  \textsuperscript{$\dag$}\{kccai, cslui\}@cse.cuhk.edu.hk,
  \textsuperscript{$\ddag$}longbohuang@tsinghua.edu.cn}
 
 \maketitle 

 \begin{abstract}
   The web link selection problem is to select a small subset of web links from
   a large web link pool, and to place the selected links on a web page that can
   only accommodate a limited number of links, e.g., advertisements,
   recommendations, or news feeds.
   %
   %
   %
   Despite the long concerned click-through rate which reflects the
   attractiveness of the link itself, the revenue can only be obtained
   from user actions after clicks, e.g., purchasing after being directed to the 
   product pages by recommendation links. 
   Thus, the web links have an intrinsic \emph{multi-level feedback structure}. 
   With this observation, we consider the context-free web link selection
   problem, where the objective is to maximize revenue while ensuring that the 
   attractiveness is no less than a preset threshold.
   The key challenge of the problem is that each link's multi-level feedbacks
   are stochastic, and unobservable unless the link is selected. 
   We model this problem with a constrained stochastic multi-armed bandit
   formulation, and design an efficient link selection algorithm, called
   Constrained Upper Confidence Bound algorithm (\myucb), and prove
   $O(\sqrt{T\ln T})$ bounds on both the regret and the violation of the attractiveness
   constraint. 
   We conduct extensive experiments on three real-world datasets, and show that
   \myucb{} outperforms state-of-the-art context-free bandit algorithms
   concerning the multi-level feedback structure.
\end{abstract}

\input{sec1-introduction}

\input{sec2-model}
\input{sec3-alg}
\input{sec4-analysis}
\input{sec5-experiments}
\vspace{-0.05in}
\input{sec6-conclusion}
\vspace{-0.05in}
\section*{Acknowledgment}
This work is supported in part by the National Natural Science Foundation of China Grants 61672316, 61303195, the Tsinghua Initiative Research Grant, and the China Youth 1000-Talent Grant.

\bibliographystyle{named}
\bibliography{bandit-ref}
\input{appendices}

\end{document}

%% file: sec1-introduction.tex
\section{Introduction}
\label{sec:intro}

\noindent With the rapid development of the Internet, web links
are playing important roles in websites and mobile apps for attracting users and 
generating revenues. 
For example, e-commerce websites such as Amazon and Taobao show featured
recommendation links on shopping pages to induce more purchase.
Online social networks such as Facebook and Google+ constantly push links of
trending topics and friends' updates to users, so as to increase user
engagement. 
Online media such as HBO and iQIYI present links to popular TV shows and movies
on their homepages to attract more subscribers.

Due to the limited space of interest on a web page, only a finite number of
links can be shown to a user when the page is browsed. This raises the web link
selection problem, i.e., how to properly select a small subset of web links from
a large link pool for a web page.
Specifically, the web page on which the selected links are shown is called the
 target page. 
If clicked, each link directs the user to a  subsequent page. 
This implies that web links provide a \emph{multi-level feedback} to the web
operator.
The  first level feedback refers to the likelihood that a
user clicks a link, also known as the Click-Through Rate (CTR) at the target page. 
By tracking user actions after clicking a link, e.g., purchase or subscription,
we can determine the revenue collected on the associated subsequent
page, which gives the  second level feedback. 
Since purchase or subscription can only happen after the click, the 
\emph{compound feedback} is the product of the first-level and the
second-level feedbacks.
Intuitively, the first-level feedback (CTR) indicates the \emph{attractiveness}
of the link at the target page, while the second-level feedback indicates the
\emph{potential revenue} that can be collected from the subsequent page. 
The compound feedback reveals the compound revenue a web link can bring. 

%
There has been a lot of research concerning the CTR of web links,
e.g., \cite{langheinrich1999unintrusive,lohtia2003impact}. On the other hand,
what happens after \blue{clicks} is also worth great attention, as it generates
revenue. For instance, cost per acquisition (CPA) is regarded as the optimal way
for an advertiser to buy online advertising \cite{CPA}. Both the attractiveness
and the profitability of a website or an app are important measures~\cite{Kohavi2014KDD}, because they represent the
long-term and short-term benefits, respectively.
This motivates us to move beyond CTR and to pursue both  attractiveness 
and profitability simultaneously in link selection. 

In this work, we consider the problem of selecting a finite number of links from
a large pool for the target page, so as to maximize the total compound revenue,
while keeping the total attractiveness above a certain threshold.
The constraint on attractiveness (CTR) was also adopted in the literate of
online advertising \cite{kumar2015ad,mookerjee2016ad}.
In addition, we also take into consideration the fact that contextual
information, e.g., user preferences, is not always available, e.g.,
incognito visits~\cite{aggarwal2010USENIX}, cold start~\cite{elahi2016survey},
or cookie blocking~\cite{meng2016trackmeornot}. 
Thus, we do not assume any prior contextual information. 
We refer to our problem as the \emph{context-free} web link selection problem. 

Our link selection problem is challenging. 
First of all, the attractiveness and profitability of a link can be conflicting goals, as shown in~\cite{cai2017}. 
As a result, while selecting links with high CTRs satisfies the
attractiveness constraint, it does not necessarily guarantee that
the target page will have a high total compound revenue, and vice versa. 
What further complicates the problem is that the multi-level feedbacks, i.e.,
the CTR (first-level feedback) and the potential revenue (second-level feedback)
of each link, are stochastic and unobservable unless a link is selected and
shown on the target page. 

We formulate our problem as a constrained stochastic multiple-play multi-armed 
bandit problem with multi-level rewards.
Specifically, there are multiple arms in the system. 
Each arm represents a link in the pool. 
Its first-level reward, second-level reward, and compound reward correspond to
the first-level feedback (the CTR), the second-level feedback (the potential
revenue) and the compound feedback (the compound revenue) of that link,
respectively.  
The objective is to select a finite number of links at each time step to
minimize the cumulative regret, as well as the cumulative violation of the
constraint. 
We design a constrained bandit algorithm, Constrained Upper Confidence Bound 
algorithm (\myucb{}), to simultaneously achieve {\em sub-linear} regret and
violation bounds. 

Our main contributions are as follows. (i)~We formulate the link selection
problem as a constrained bandit problem with stochastic multi-level rewards
(Section~\ref{sec:model}). (ii)~We propose the \myucb{} algorithm
(Section~\ref{sec:alg}) and prove that \myucb{} 
ensures \blue{small} regret and violation bounds with high probability, i.e., for any given  failure 
probability $\delta\in(0,1)$, the regret and violation at time $T$ are bounded
by $O(\sqrt{T\ln\frac{T}{\delta}})$ with probability at least $1-\delta$
(Section~\ref{sec:analysis}). (iii)~We conduct extensive experiments on three
real-world datasets. Our results show that \myucb{} outperforms three
state-of-the-art context-free bandit algorithms,
\textsc{CUCB}~\cite{chen2013combinatorial},
\textsc{Exp3.M}~\cite{uchiya2010algorithms}, and \textsc{LExp}~\cite{cai2017} for 
the constrained link selection problem
(Section~\ref{sec:experiments}). 

\section{Related work}
\label{sec:related-work}
Link selection, or website optimization, has long been an important problem. 
One common approach for the problem is A/B
testing~\cite{Xu:2015:kdd:abtest,Deng2017WSDM}, which splits the traffic to two
web pages with different designs, and evaluates their performances.
However, the overhead of A/B testing can be high when the web link pool is
large, as it needs to compare different link combinations. Moreover, A/B testing does
not have any loss/regret guarantees. 
Another approach is to model the link selection problem as a contextual multi-armed bandit
problem~\cite{li2010contextual}, and to incorporate the collaborative filtering
method \cite{collfilteringSIGMETRICS16,CFBanditSIGIR16}. However, these
contextual bandit formulations neglect the multi-level feedback structures and
do not consider any constraint.

The multiple-play multi-armed bandit problem, where multiple arms are selected
in each round, has been studied from both theoretical and empirical
perspectives, and many policies have been
designed~\cite{uchiya2010algorithms,chen2013combinatorial,Komiyama1506,lagree2016nips}.
Our constrained multiple-play bandit model differs from aforementioned
models in that we consider meeting the constraint on the total first-level
rewards in selecting multiple arms, which is important for web link selection.

Recently, bandit with
budgets~\cite{ding2013aaai,wu2015nips,xia2016budgeted} and bandit with
knapsacks~\cite{badanidiyuru2013bandits,agrawalBwCR} have attracted much
research attention. 
In these problems, pulling an arm costs certain resources, and 
each resource has a budget. Thus,  resource cost is implicitly taken into
consideration during the analysis of regret in the above two formulations since the 
arm selection process stops when resources are depleted. In contrast, since
the constraint in our model is a requirement on the average performance, our
arm selection procedure can last for an arbitrary length of time, and we need to consider both the regret and the violation of the constraint during the process. 
Thus, while our work builds upon the results in~\cite{badanidiyuru2013bandits} and
\cite{agrawalBwCR}, the problem is different, and we study the multiple-play case
rather than the single-play case. 
In addition, we conduct experiments on real-world datasets, which are not
included in their works.
On the other hand, the thresholding bandit problem in~\cite{Locatelli1605a} is
to find the set of arms whose means are above a given threshold through pure
exploration in a fixed time horizon, which is different from our model.


\blue{Our work is closest to recent work~\cite{cai2017}. They assume the second-level reward is adversarial.
However, it has been observed that this might not be the case in practice~\cite{pivazyan2004thesis}, 
and user behavior is likely to follow certain statistical rules when the number of users is large. So we study the stochastic case. 
Most importantly, our algorithm guarantees performance with high probability rather than in expectation, 
and the regret and violation bounds are improved significantly from $O(T^{\frac{2}{3}})$ and $O(T^{\frac{5}{6}})$ 
in their algorithm (\lexp{}) to both $O(\sqrt{T\ln T})$ in our algorithm (\myucb{}).}

%% file: sec2-model.tex
\section{Model}
\label{sec:model}

\noindent 


Consider the two-level feedback context-free web link selection problem, where
one needs to select $L$ links from a pool of $K$ web links, $\{l_{1},\ldots,
l_{K}\}, L \le K$, to display on the target page. 
Each link directs users to a subsequent page. 
If $l_{i}$ is shown on the target page, we obtain the following feedbacks
when users browse the page: 
\begin{enumerate}
\item the click-through rate (CTR), i.e., the probability that a user clicks
  $l_{i}$ to visit the corresponding  subsequent page,
\item the after-click revenue, i.e., the revenue collected from each user who
  clicks $l_{i}$ and then purchases products (or subscribes to programs) on the corresponding 
  subsequent page.
\end{enumerate}
In practice, the click-through rate and the after-click revenue are stochastic,
and we do not assume any prior knowledge about their distributions or
expectations.
The product of the CTR and the after-click revenue is the compound revenue, i.e.,
the revenue that $l_{i}$ can bring if it is shown on the target page. 
The objective of the link selection problem is to maximize the total compound
revenue of the selected $L$ links, subject to the constraint that the total CTR
of these selected links is no less than a preset threshold $h>0$,\footnote{CTR
  measures the attractiveness of a link to users and is an important metric for
  the link selection problem.} 
where $h$ is determined by the web operator based on  service requirement.
An example is that in online advertising, the constraint on CTR is usually specified in the contract between the publisher (web operator)
and the advertising firm \cite{kumar2015ad,mookerjee2016ad}.

To address the link selection problem, we formulate it as a constrained
stochastic multi-armed bandit problem with multiple plays, where each arm has a 
two-level reward structure. 
In this formulation, \blue{each time step is a short duration} and each arm corresponds to a specific web link. Thus, the set
of arms can be written as $\mathcal{K}=\{1,\ldots, K\}$. Each arm $i$ is
associated with two sequences of random variables, $\{a_i^{t}\}_{t=1}^T$ and $\{b_i^{t}\}_{t=1}^T$, where  $a_i^{t}$ characterizes arm $i$'s \emph{first-level reward} (CTR) at time $t$, and  
$b_{i}^{t}$ characterizes arm $i$'s \emph{second-level reward} (after-click revenue).  
We assume that
for any $i\in\mathcal{K}$, both $\{a_i^{t}\}_{t=1}^T$ and $\{b_i^{t}\}_{t=1}^T$ are sequences of i.i.d. random variables.
The expectations of $a_{i}^{t}$ and $b_{i}^{t}$ are  denoted by 
$a_{i}=\Ex[a_{i}^{t}]$ and $b_{i}=\Ex[b_{i}^{t}]$, $ i \in \mathcal{K}$. We also
assume that $a_{i}^{t}$ is independent of $b_{i}^{t}$ for $i\in \mathcal{K}$,
$t\ge 1$. Thus, the compound reward of arm $i$ at time $t$ is
$g_{i}^{t}=a_{i}^{t}b_{i}^{t}$ with mean $g_{i}=\Ex[g_{i}^{t}]=a_{i}b_{i}$.
%
Denote $\bm{a} = (a_{1},\ldots, a_{K})$ and $\bm{g} = (g_{1},\ldots, g_{K})$.
Without loss of generality, we assume that $a_{i}^{t}\in [0,1]$ and
$b_{i}^{t}\in [0,1]$.

%
As mentioned above, the distributions or expectations of the two-level reward 
for any arm are unknown beforehand. At each time step $t$, an algorithm $\pi$ 
selects a set of $L\le K$ arms $\mathcal{I}_{t}(\pi)\subset \mathcal{K}$, and
observes the \emph{first level reward} $a_{i}^{t}$ as well as the \emph{second
  level reward} $b_{i}^{t}$ for each arm $i\in \mathcal{I}_{t}(\pi)$. 
  The
optimal policy is the one that maximizes the expected total compound reward of 
the selected $L$ arms, while keeping the total first level reward above the
preset threshold $0<h<L$.\footnote{If $h=0$, the problem is equivalent to the
  classic unconstrained multiple-play multi-armed bandit problem
  (MP-MAB)~\cite{anantharam1987asymptotically}. If $h\ge L$, there is no
  policy that can satisfy the constraint.}

The optimal policy is not limited to deterministic policies as in traditional
multi-armed bandit problems \cite{auer2002finite,bubeck2012regret}, but can be
randomized, i.e., a distribution on the possible selections $\mathcal{I}_{t}$.
In practice, the number of web links $K$ can be very large, and the number of
possible selections of links at each time step can be as large as
$\binom{K}{L}$, which makes it complicated to consider randomized policies.
To simplify the problem, we represent a randomized policy with a
\emph{probabilistic selection vector} $\bm{x}=(x_{1},\ldots,x_{i},\ldots,
x_{K}),\bm{1}^{\intercal}\bm{x} = L$, where $x_{i}\in[0,1]$ is the probability
of selecting arm $i$ and $\bm{1}=(1,\ldots,1)$ is the one vector.\footnote{If not specified otherwise, all vectors defined
  in this paper are column vectors.} 
At each time $t$, the selection set $\mathcal{I}_{t}(\bm{x})$ under a randomized
policy $\bm{x}$ is generated via a dependent rounding
procedure~\cite{gandhi2006dependent}, which guarantees the probability that
$i\in\mathcal{I}_{t}(\bm{x})$ is $x_{i}$ (see Section~\ref{sec:alg}). 

The set of randomized policies can be denoted by
$\mathcal{X}=\{\bm{x}\in\mathbb{R}^{K}|0\le x_i\le 1,
\bm{1}^{\intercal}\bm{x}=L\}$. 
Thus, the optimal stationary randomized policy is
\begin{equation}
\bm{x}^{*}=\argmax_{\bm{x}^{\intercal}\bm{a}\ge h}\bm{x}^{\intercal}\bm{g}. \label{eq:optimal-policy}
\end{equation}

Our objective is to design an algorithm $\pi$ to decide the selection set
$\mathcal{I}_{t}(\pi)$ for $t=1,\ldots,T$, such that the \emph{regret}, i.e.,
the accumulated difference $\textmd{Reg}_{\pi}(T) $ between the compound reward
under $\pi$ and that under the optimal policy, is minimized. Specifically,
\begin{equation}
\label{eq:regret-def}
\textmd{Reg}_{\pi}(T) = T{\bm{x}^{*}}^{\intercal}\bm{g}-\sum_{t=1}^T\sum_{i\in \mathcal{I}_{t}(\pi)}g_i^{t}.
\end{equation} 
Note that the total first-level reward of arms in $\mathcal{I}_{t}(\pi)$ may
violate the constraint, especially when $t$ is small and we have little
information about the arms. To measure the overall violation of the constraint
at time $T$, we define \emph{violation} of algorithm $\pi$ as,
\begin{equation}
\label{eq:violation-def}
\textmd{Vio}_{\pi}(T) =[hT-\sum_{t=1}^T\sum_{i\in \mathcal{I}_{t}(\pi)}a_i^{t}]_+,
\end{equation}
where $[x]_{+}=\max(x,0)$. Note that when designing link selection algorithms,
we should take both the regret and violation into consideration, so as to
achieve both {\em sub-linear} regret and {\em sub-linear} violation with respect
to $T$.
Also, note that our model can be generalized to link selection problems with
$n$-level ($n>2$) feedback structures, by taking a subsequent page as a new
target page and select links for it with the above model, and so on.

%% file: sec3-alg.tex
\section{Algorithm}
\label{sec:alg}

\noindent
In this section, we present our Constrained Upper Confidence Bound algorithm
(\myucb{}), and describe its details in Algorithm \ref{alg:myucb}. 
Let $H_{t}=\{\mathcal{I}_{\tau}, a_{i}^{\tau},b_{i}^{\tau}: i\in
\mathcal{I}_{\tau}, 1\le \tau\le t\}$ denote the historical information of
chosen actions and observations up to time $t$. 
Define the empirical average first-level reward and compound reward for each arm
$i$ as
\begin{equation} \label{eq:average}
\begin{aligned}
\bar{a}_{i}^{t}=\frac{\sum_{\tau<t,i\in\mathcal{I}_{\tau}}a_{i}^{\tau}}{N_{i}^{t}+1},\\ \bar{g}_{i}^{t}=\frac{\sum_{\tau<t,i\in\mathcal{I}_{\tau}}g_{i}^{\tau}}{N_{i}^{t}+1},
\end{aligned}
\end{equation}
where $N_{i}^{t}$ is the number that arm $i$ is played before time $t$. 
Define $R(\mu,n)=\sqrt{\frac{\gamma \mu}{n}}+\frac{\gamma}{n}$ as   in
\cite{kleinberg2008bandits} where $\gamma$ is a constant. 
In \myucb{}, we use the following Upper Confidence Bounds for the unknown
rewards \cite{agrawalBwCR}:
\begin{align*}
\hat{a}_{i}^{t}=&\min\{1,\bar{a}_{i}^{t}+2R(\bar{a}_{i}^{t},N_{i}^{t}+1)\},\\
\hat{g}_{i}^{t}=&\min\{1,\bar{g}_{i}^{t}+2R(\bar{g}_{i}^{t},N_{i}^{t}+1)\}.
\end{align*}
Denote $\bar{\bm{a}}^{t} = (\bar{a}_{1}^{t},\ldots, \bar{a}_{K}^{t})$,
$\bar{\bm{g}}^{t} = (\bar{g}_{1}^{t},\ldots, \bar{g}_{K}^{t})$, and
$\hat{\bm{a}}^{t} = (\hat{a}_{1}^{t},\ldots, \hat{a}_{K}^{t})$,
$\hat{\bm{g}}^{t} = (\hat{g}_{1}^{t},\ldots, \hat{g}_{K}^{t})$. 
In the initialization step of Algorithm~\ref{alg:myucb}, $\gamma$ is set
  to $72\ln\frac{8KT}{\delta}$, where $\delta\in(0,1)$ is an input parameter, i.e., the allowed failure probability.

Specifically, in each round, \myucb{} solves the optimization problem
\eqref{eq:myucb:optimization} to get the probabilistic selection vector
$\bm{x}_{t}$ (line~\ref{alg:myucb:optimization}). Notice that
\eqref{eq:myucb:optimization} is similar to the original constrained optimization
problem \eqref{eq:optimal-policy} but uses the Upper Confidence Bounds to
replace the unknown rewards. Then,  $\mathcal{I}_{t}$ is generated via a dependent 
rounding procedure. In line~\ref{alg:myucb:estimate-abg} we receive the
two-level rewards $a_{i}^{t}$ and $b_{i}^{t}$ for arms in $\mathcal{I}_t$ and
update the empirical average rewards to get the Upper Confidence Bounds for the next round.

\begin{algorithm}[t]
\caption{Constrained Upper Confidence Bound}
\label{alg:myucb}
\begin{algorithmic}[1]
\Require $K$, $L$, $h$, $\delta\in(0,1)$.
\Ensure Selected arm set for each round.
\Initial Set  $\gamma=72\ln\frac{8KT}{\delta}, \bar{\bm{g}}^{1}=\bm{0}, \bar{\bm{a}}^{1}=\bm{0}$, and $N_{i}^{1}=0,\forall i$. 
\For{$t=1, 2,\dots,T$}
\State \label{alg:myucb:optimization} Solve the following linear optimization
problem: 
\begin{equation}
\bm{x}_{t}={\arg\max}_{\bm{x}^{\intercal}\hat{\bm{a}}^{t}\ge h,\bm{x}\in\mathcal{X}}\bm{x}^{\intercal}\hat{\bm{g}}^{t}. \label{eq:myucb:optimization}
\end{equation}
\hspace{0.166in} If \eqref{eq:myucb:optimization} has no feasible solution, set $\bm{x}_{t}\in\mathcal{X}$ arbitrarily.
\State  Set $\mathcal{I}_t=\text{\textsc{DependentRounding}}(L,\bm{x}_{t}).$
\State \label{alg:myucb:estimate-abg} Receive $a_{i}^{t}$ and $b_{i}^{t}$ for $i\in \mathcal{I}_t$. Update
\begin{align*}
N_{i}^{t+1}&=
\begin{cases}
N_{i}^{t}+1,  \text{ $i\in \mathcal{I}_t$}, \\
N_{i}^{t},  \text{\qquad$i\notin \mathcal{I}_{t}$},
\end{cases} \\
\bar{g}_{i}^{t+1}&=
\begin{cases}
[\bar{g}_{i}^{t}(N_{i}^{t}+1)+g_{i}^{t}]/(N_{i}^{t+1}+1),  &\text{$i\in \mathcal{I}_t$}, \\
\bar{g}_{i}^{t},  & \text{$i\notin \mathcal{I}_{t}$},
\end{cases} \\
\bar{a}_{i}^{t+1}&=
\begin{cases}
[\bar{a}_{i}^{t}(N_{i}^{t}+1)+a_{i}^{t}]/(N_{i}^{t+1}+1), &\text{$i\in \mathcal{I}_t$}, \\
\bar{a}_{i}^{t},  &\text{$i\notin \mathcal{I}_{t}$}.
\end{cases}
\end{align*}
\EndFor
\Function{DependentRounding}{$L, \bm{x}$}\label{alg:myucb:func-dr}
\While  {exists $i$ such that $0<x_{i}<1$}
\State {Find $i,j, i\neq j$, such that $x_{i},x_{j}\in(0,1)$.}
\State {Set $p=\min\{1-x_{i},x_{j}\}$,
$q=\min\{x_{i},1-x_{j}\}$.}
\State Update $x_{i}$ and $x_{j}$ as
\State$
  (x_{i},x_{j})=
\begin{cases}
(x_{i}+p, x_{j}-p),  \text{ probability } \frac{q}{p+q}; \\
(x_{i}-q, x_{j}+q),  \text{ probability } \frac{p}{p+q}.
\end{cases}
$
\EndWhile
\State \Return $\mathcal{I}=\{i\,|\,x_{i}=1, 1\le i\le K\}$.
\EndFunction
\end{algorithmic}
\end{algorithm}

%% file: sec4-analysis.tex
\section{Theoretical Analysis}
\label{sec:analysis}

\noindent In this section, we bound the regret and violation of Algorithm \ref{alg:myucb}. 
We will make use of the concentration
  inequalities in the following lemmas.
\begin{lemma}[Azuma-Hoeffding inequality~\cite{azuma1967weighted}]\label{lem:azuma}
  Suppose $\{Y_{n}: n= 0, 1, 2, 3, \dots\}$ is a martingale and
  $|Y_{n}-Y_{n-1}|\le c_{n}$ almost surely, then with probability at least
  $1-2e^{-\frac{d^{2}}{2\sum_{j=1}^{n}c_{j}^{2}}}$, we have
\begin{equation*}
|Y_{n}-Y_{0}|\le d.
\end{equation*}
\end{lemma}
\begin{lemma}[\cite{kleinberg2008bandits,badanidiyuru2013bandits,agrawalBwCR}]\label{lem:ucb-radius}
  Consider $n$ i.i.d random variables $Z_{1},\dots,Z_{n}$ in $[0,1]$ with
  expectation $z$. Let $\mu$ denote their empirical average. Then, for any
  $\gamma>0$, with probability at least $1-2e^{-\frac{1}{72}\gamma}$, we have
\begin{equation*}
|\mu-z|\le R(\mu,n),
\end{equation*}
where $R(\mu,n)=\sqrt{\frac{\gamma \mu}{n}}+\frac{\gamma}{n}$.
\end{lemma}
The following lemma is a corollary of Lemma \ref{lem:ucb-radius}. 
\begin{lemma}\label{lem:ucb}
  Define the empirical averages $\bar{a}_{i}^{t}$ and $\bar{g}_{i}^{t}$ as in
  \eqref{eq:average}. Then, for every $i$ and $t$, with probability at least
  $1-2e^{-\frac{1}{72}\gamma}$, we have
\begin{equation*}
|\bar{a}_{i}^{t}-a_{i}|\le 2R(\bar{a}_{i}^{t},N_{i}^{t}+1),
\end{equation*}
where $\gamma\ge1$. The same result holds between $\bar{g}_{i}^{t}$ and $g_{i}$.
\end{lemma}
\begin{proof}
  For every $i$ and $t$, applying Lemma \ref{lem:ucb-radius}, we have that, with
  probability at least $1-2e^{-\frac{1}{72}\gamma}$,
\begin{align*}
|\frac{N_{i}^{t}+1}{N_{i}^{t}}\bar{a}_{i}^{t}-a_{i}|&\le R(\frac{N_{i}^{t}+1}{N_{i}^{t}}\bar{a}_{i}^{t},N_{i}^{t}), \\
|\bar{a}_{i}^{t}-a_{i}+\frac{a_{i}}{N_{i}^{t}+1}|&\le \frac{N_{i}^{t}}{N_{i}^{t}+1}R(\frac{N_{i}^{t}+1}{N_{i}^{t}}\bar{a}_{i}^{t},N_{i}^{t}). 
\end{align*}
This implies that 
\begin{align*}
|\bar{a}_{i}^{t}-a_{i}|&\le \frac{N_{i}^{t}}{N_{i}^{t}+1}(\sqrt{\frac{\gamma (N_{i}^{t}+1)\bar{a}_{i}^{t}}{N_{i}^{t}\cdot N_{i}^{t}}}+\frac{\gamma}{N_{i}^{t}})+\frac{a_{i}}{N_{i}^{t}+1}, \\
&= R(\bar{a}_{i}^{t},N_{i}^{t}+1)+\frac{a_{i}}{N_{i}^{t}+1}, \\
&\le 2R(\bar{a}_{i}^{t},N_{i}^{t}+1).
\end{align*}
The last inequality holds because $a_{i}\le1\le\gamma$.
\end{proof}
Based on the above lemmas, we obtain the following properties about \myucb{}.
\begin{lemma}\label{lem:reward-bound}
  By running \myucb{} for $T$ rounds with $\gamma=72\ln\frac{8KT}{\delta}$, with
  probability at least $1-\delta$, the following results hold simultaneously:
\begin{align}
&a_{i}\le\hat{a}_{i}^{t},\forall 1\le i\le K, \forall 1\le t\le T, \label{eq:a-upper}\\
&|\sum_{t=1}^{T}(\sum_{i\in\mathcal{I}_{t}}a_{i}^{t}-\bm{x}_{t}^{\intercal}\hat{\bm{a}}^{t})|=O(L\sqrt{KT\ln\frac{KT}{\delta}}), \label{eq:a-concentration}\\
&g_{i}\le\hat{g}_{i}^{t},\forall 1\le i\le K, \forall 1\le t\le T, \label{eq:g-upper}\\
&|\sum_{t=1}^{T}(\sum_{i\in\mathcal{I}_{t}}g_{i}^{t}-\bm{x}_{t}^{\intercal}\hat{\bm{g}}^{t})|=O(L\sqrt{KT\ln\frac{KT}{\delta}}). \label{eq:g-concentration}
\end{align}
\end{lemma}
\begin{proof}
See Appendix \ref{apd:proof:reward-bound}.
\end{proof}

From Lemma \ref{lem:reward-bound}, we can obtain the regret and violation bounds for \myucb{}.
\begin{theorem}\label{theo:regret-bound}
  For all $T>0$, let $\gamma=72\ln\frac{8KT}{\delta}$. By running \myucb{}, we have with probability at least $1-\delta$
  that,
\begin{align*}
\textmd{Reg}(T)=&O(L\sqrt{KT\ln\frac{KT}{\delta}}), \\
\textmd{Vio}(T)=&O(L\sqrt{KT\ln\frac{KT}{\delta}}).
\end{align*}
\end{theorem}
\begin{proof}
 We bound the regret and violation using
  \eqref{eq:a-upper} to \eqref{eq:g-concentration}, which were shown to hold with
  probability at least $1-\delta$ in Lemma \ref{lem:reward-bound}. 
  
  From \eqref{eq:a-upper} we know for all $t$, $\bm{x}^{*}$ is a feasible
  solution of the optimization problem \eqref{eq:myucb:optimization}, i.e.,
  ${\bm{x}^{*}}^{\intercal}\hat{\bm{a}}^{t}\ge {\bm{x}^{*}}^{\intercal}\bm{a}\ge
  h$. 
  Then, for all $1\le t\le T$, we have,
\begin{equation}
\bm{x}_{t}^{\intercal}\hat{\bm{g}}^{t}\ge{\bm{x}^{*}}^{\intercal}\hat{\bm{g}}^{t}\ge {\bm{x}^{*}}^{\intercal}\bm{g}, \label{eq:reward-ucb}
\end{equation}
where the last inequality follows from \eqref{eq:g-upper}. 
Combining \eqref{eq:g-concentration} and \eqref{eq:reward-ucb}, we have
\begin{equation*}
\textmd{Reg}(T)=O(L\sqrt{KT\ln\frac{KT}{\delta}}).
\end{equation*}
On the other hand, since for all $t$, \eqref{eq:myucb:optimization} has a
feasible solution $\bm{x}^{*}$, we know
$\bm{x}_{t}^{\intercal}\hat{\bm{a}}^{t}\ge h,1\le t\le T$. Then with
\eqref{eq:a-concentration}, we can get
\begin{equation*}
\textmd{Vio}(T)=O(L\sqrt{KT\ln\frac{KT}{\delta}}).
\end{equation*}
This completes the proof. 
\end{proof}



%% file: sec5-experiments.tex
\begin{figure*}[!t]
  \centering
  \subfigure[Cumulative regret]{
  \includegraphics[width=0.232\textwidth]{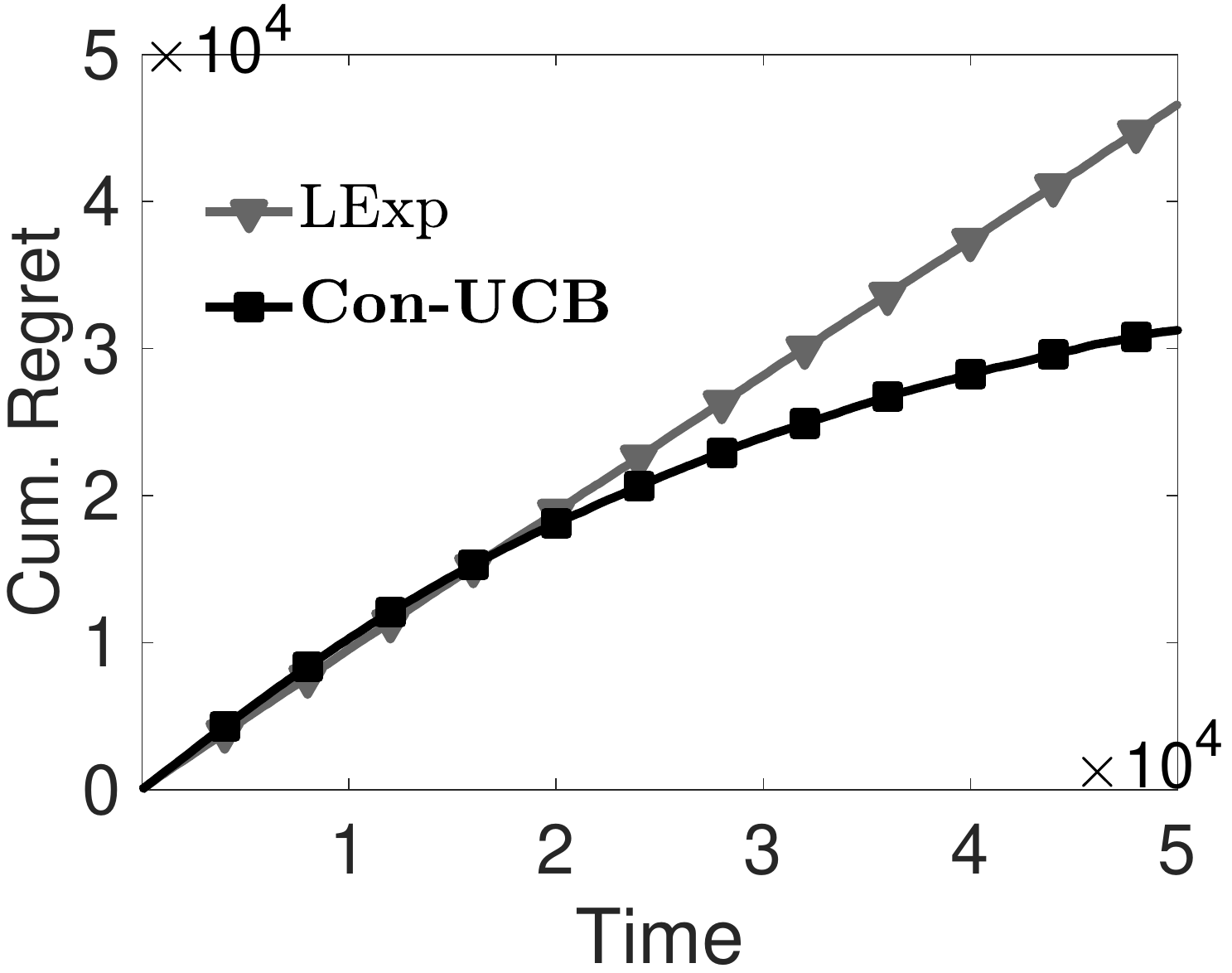}
 \label{fig:coupon:cumregret}
  }
  \subfigure[Cumulative violation]{
  \includegraphics[width=0.232\textwidth]{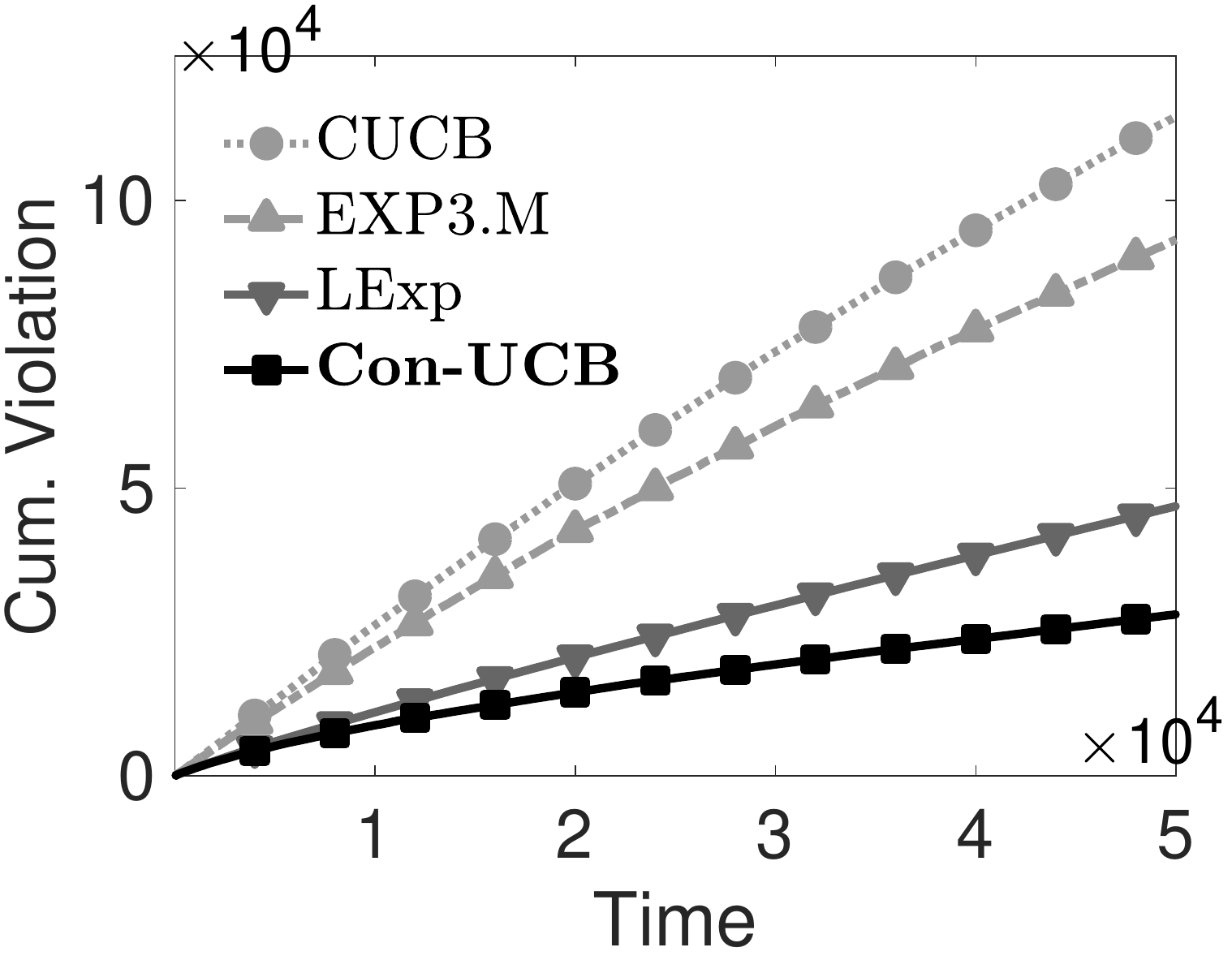}
 \label{fig:coupon:cumviolation}
  }
  \subfigure[Cum. compound reward]{
  \includegraphics[width=0.232\textwidth]{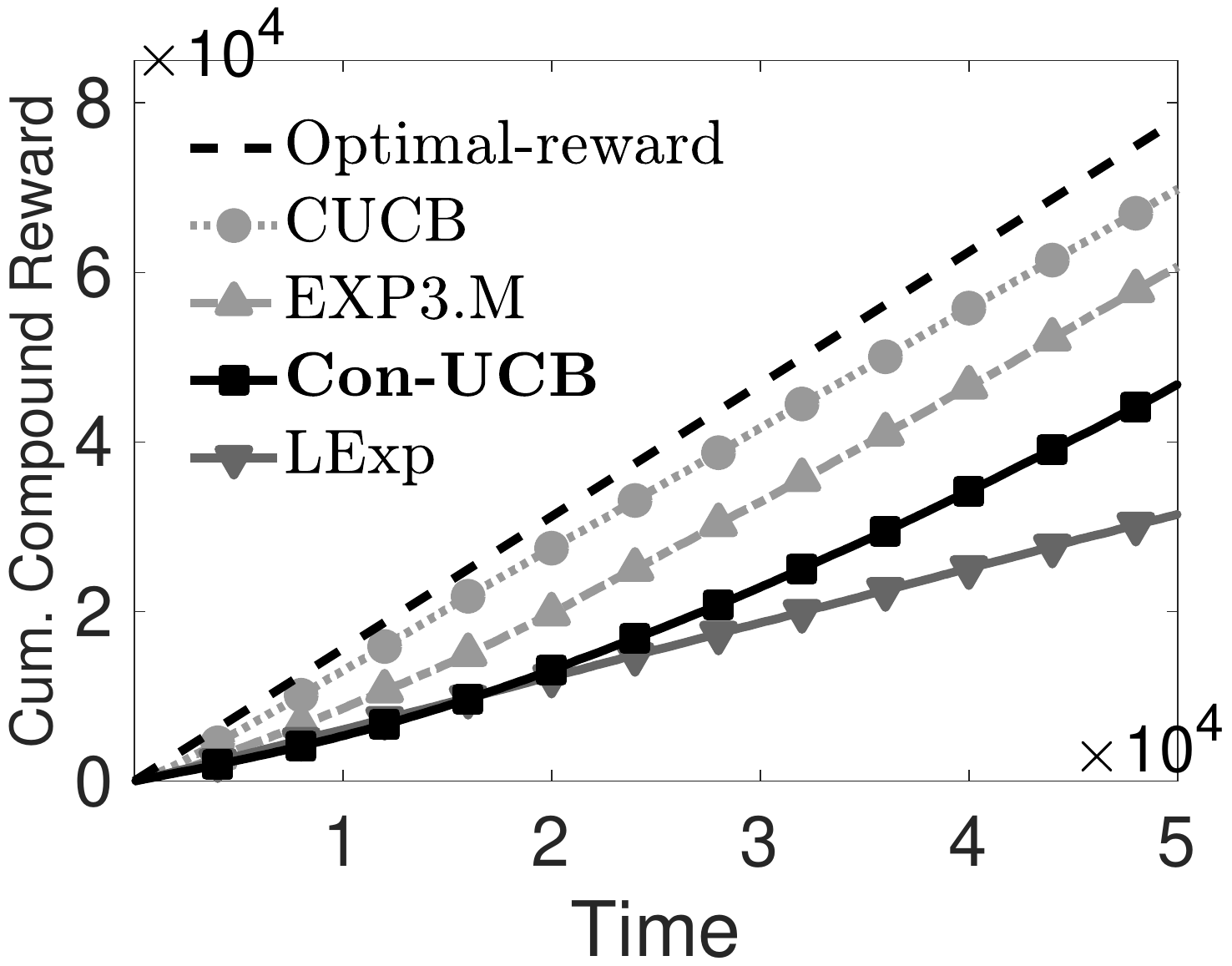}
 \label{fig:coupon:cumreward}
  }
  \subfigure[Reward/Violation ratio]{
  \includegraphics[width=0.232\textwidth]{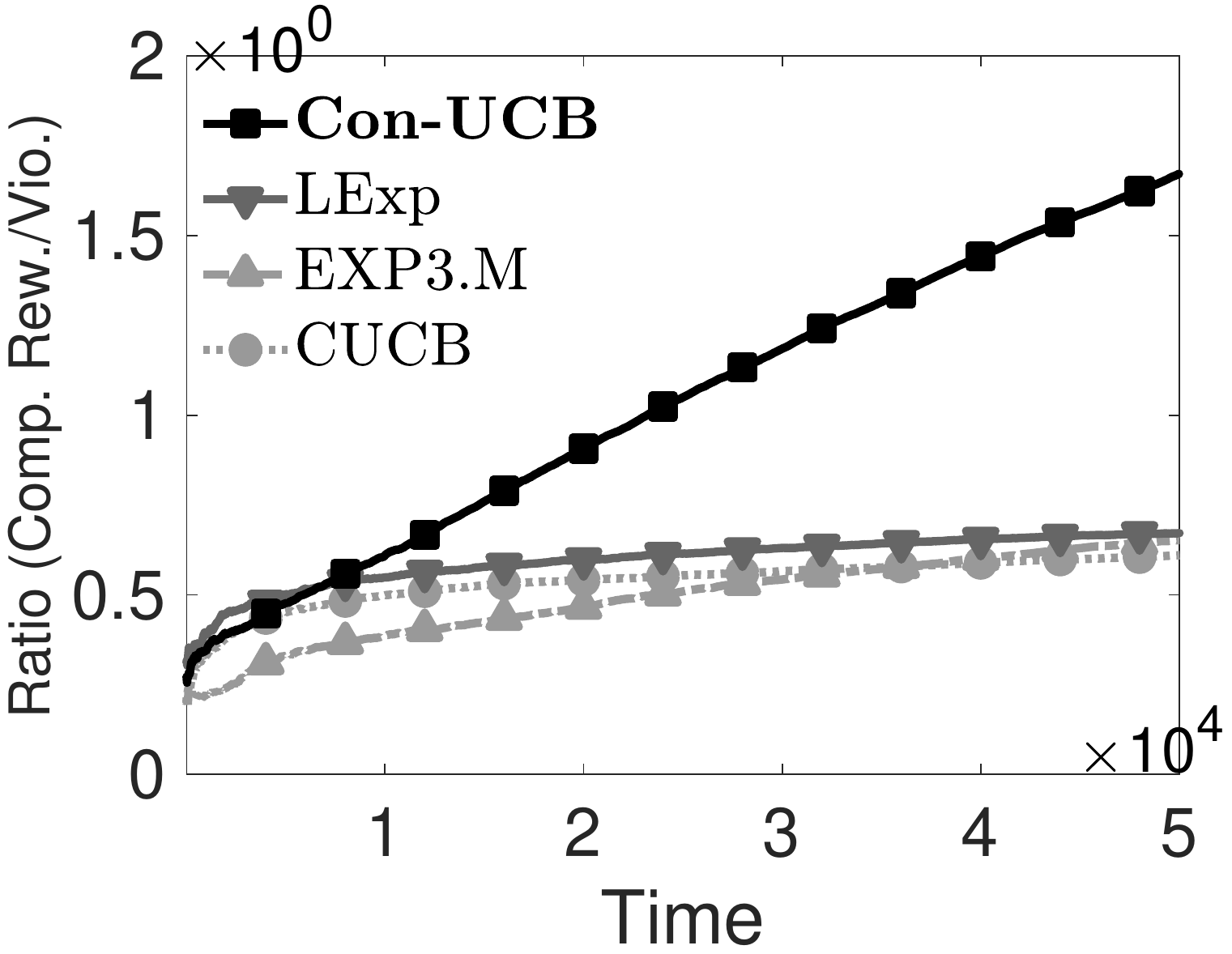}
 \label{fig:coupon:ratio}
}
\vspace{-0.1in}
\caption{Experiment results on the \coupon{} dataset. $K=271, L=15, h=4, \delta=0.01$.}
 \label{fig:coupon}
\end{figure*}
\vspace{-0.05in}

\begin{figure*}[!t]
  \centering
  \subfigure[Cumulative regret]{
  \includegraphics[width=0.232\textwidth]{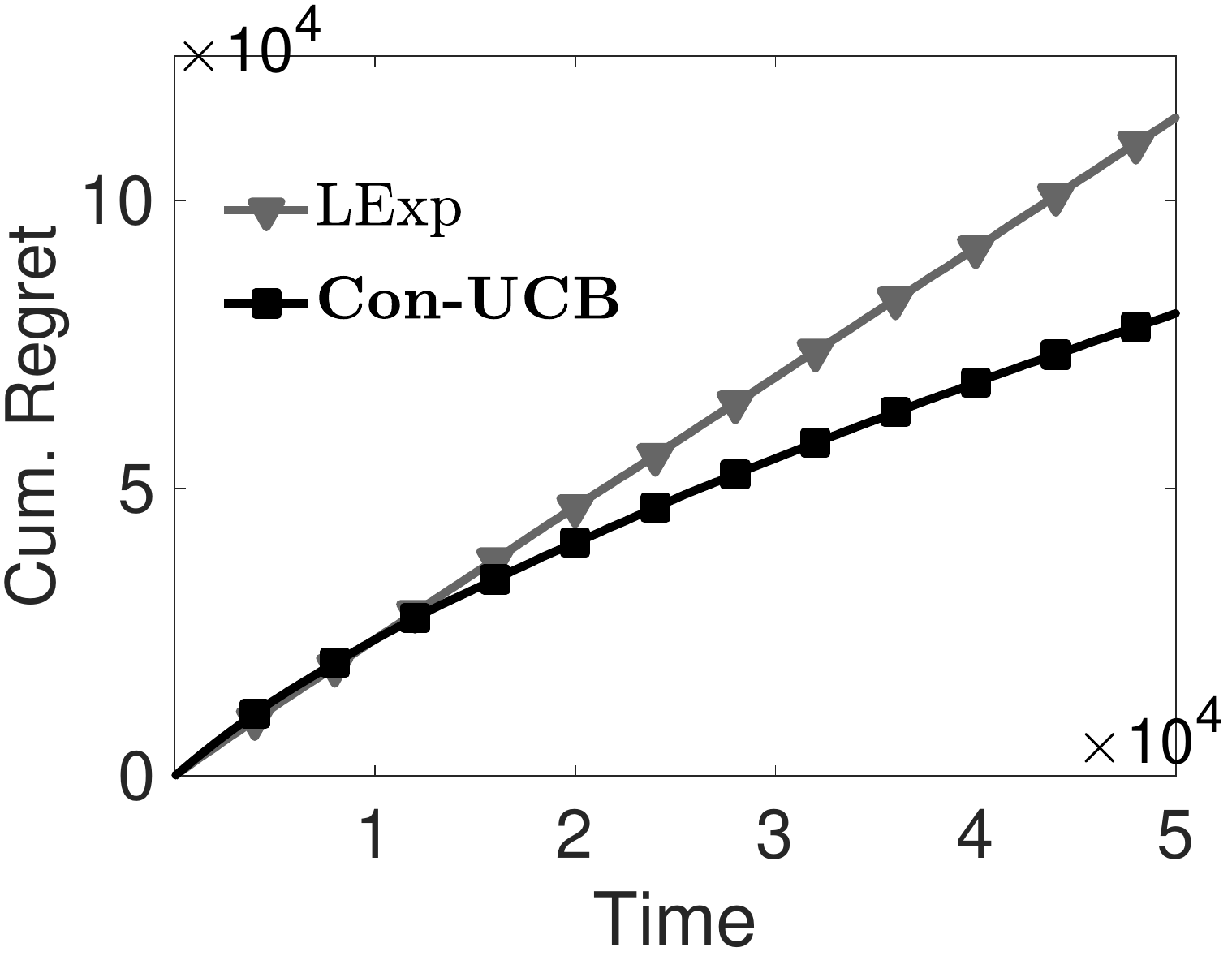}
 \label{fig:ad:cumregret}
  }
  \subfigure[Cumulative violation]{
  \includegraphics[width=0.232\textwidth]{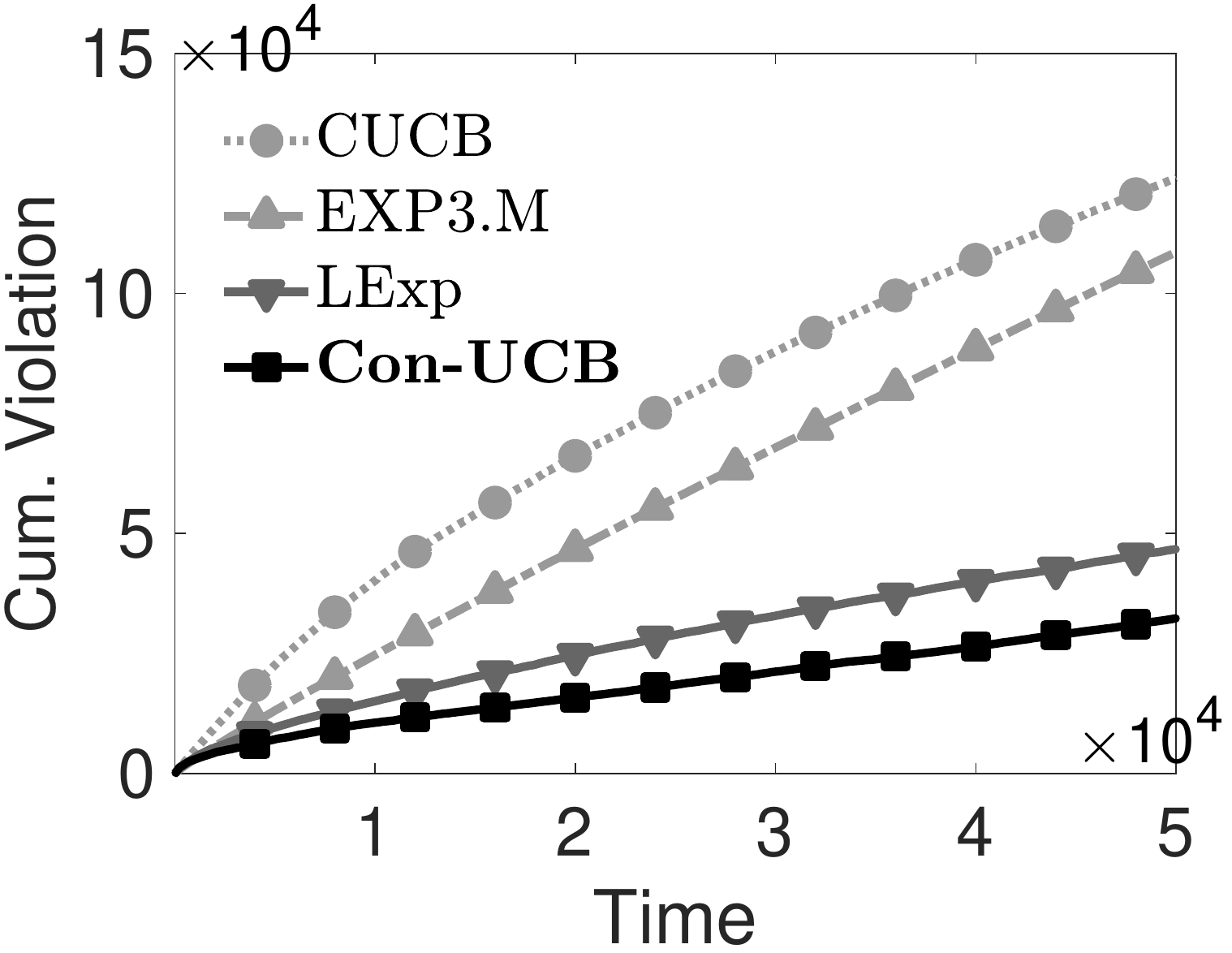}
 \label{fig:ad:cumviolation}
  }
  \subfigure[Cum. compound reward]{
  \includegraphics[width=0.232\textwidth]{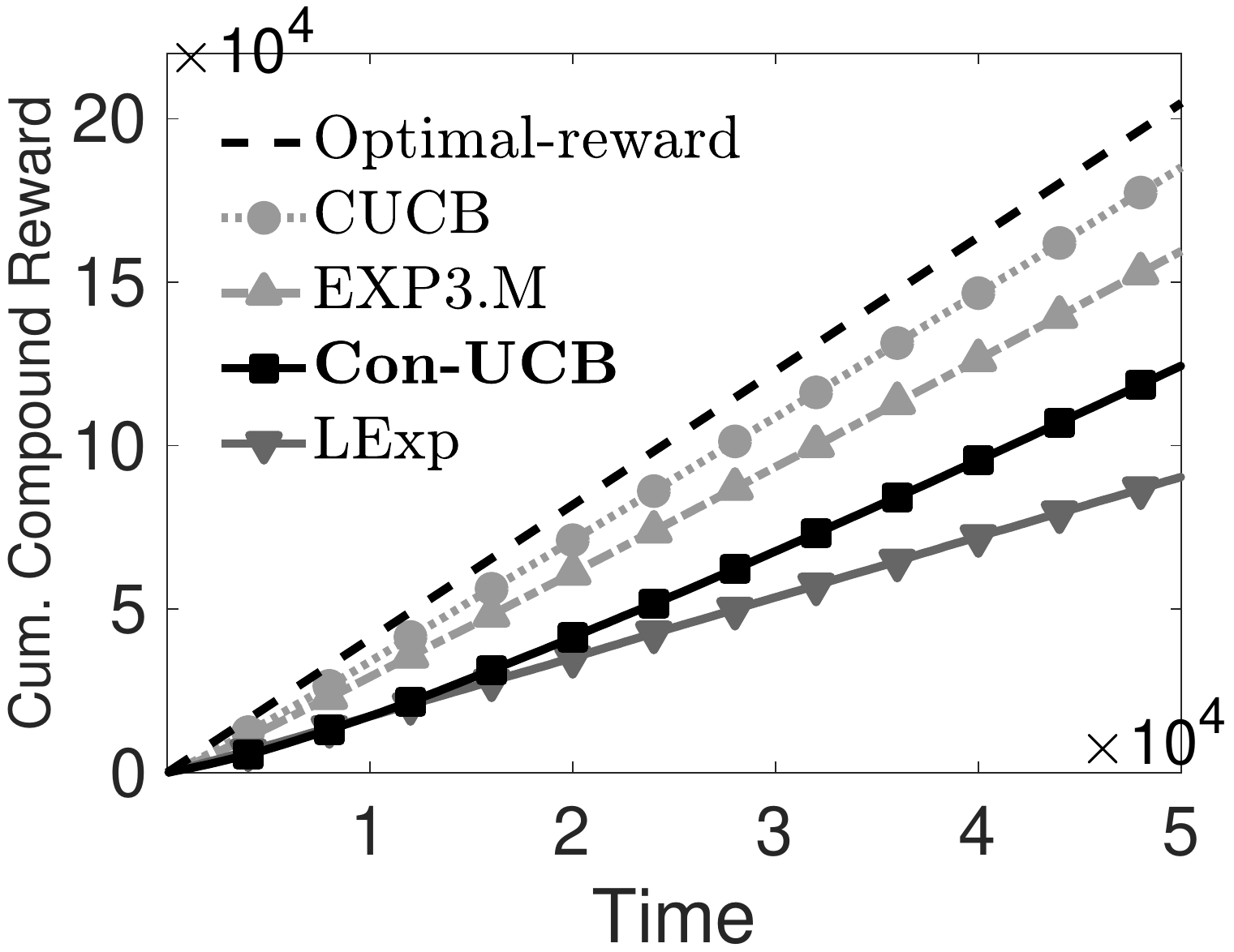}
 \label{fig:ad:cumreward}
  }
  \subfigure[Reward/Violation ratio]{
  \includegraphics[width=0.232\textwidth]{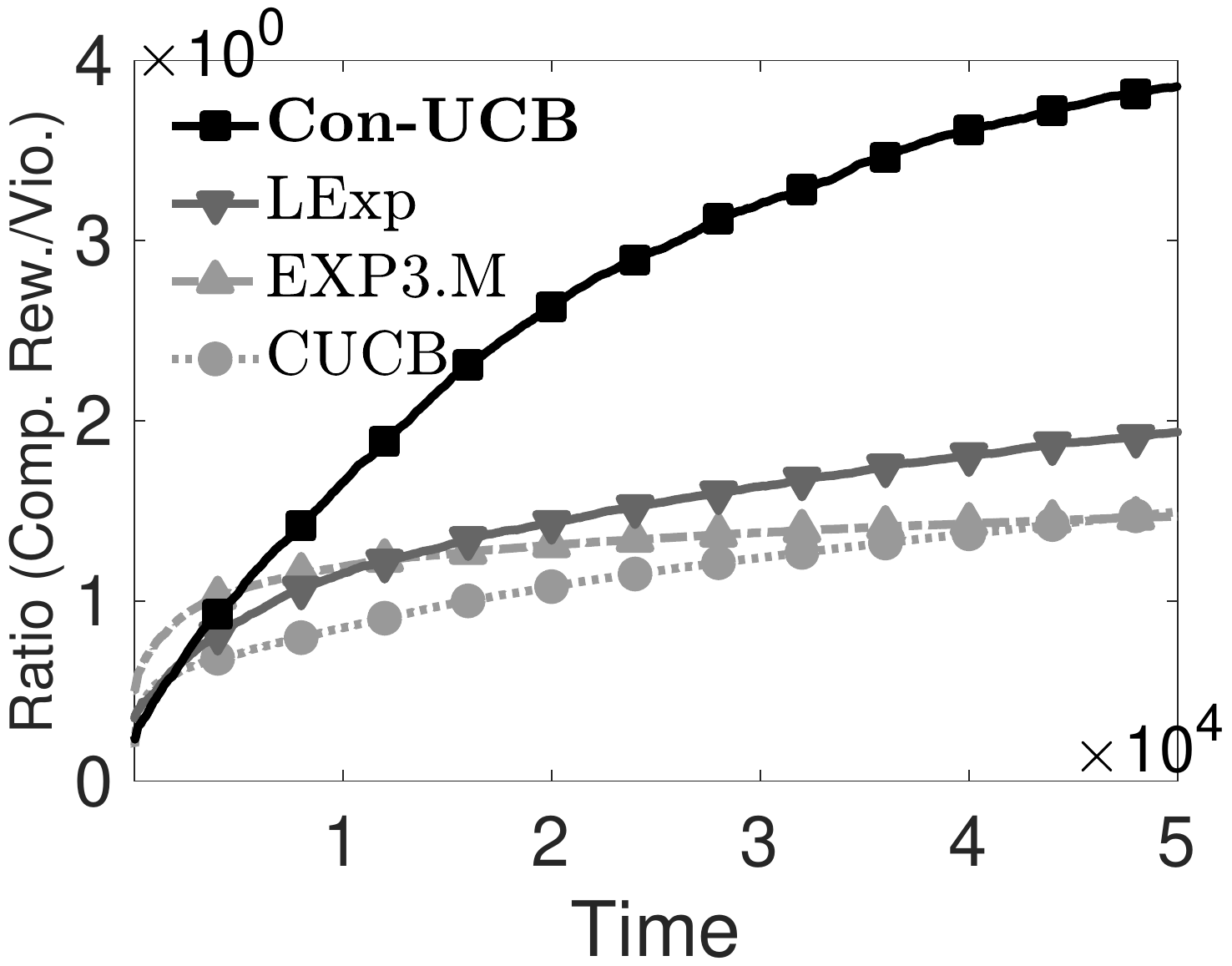}
  \label{fig:ad:ratio}
}
\vspace{-0.1in}
\caption{Experiment results on the \adclick{} dataset. $K=225, L=20, h=10, \delta=0.02$.}
 \label{fig:ad}
\end{figure*}
\vspace{-0.05in}

\begin{figure*}[!t]
  \centering
  \subfigure[Cumulative regret]{
  \includegraphics[width=0.232\textwidth]{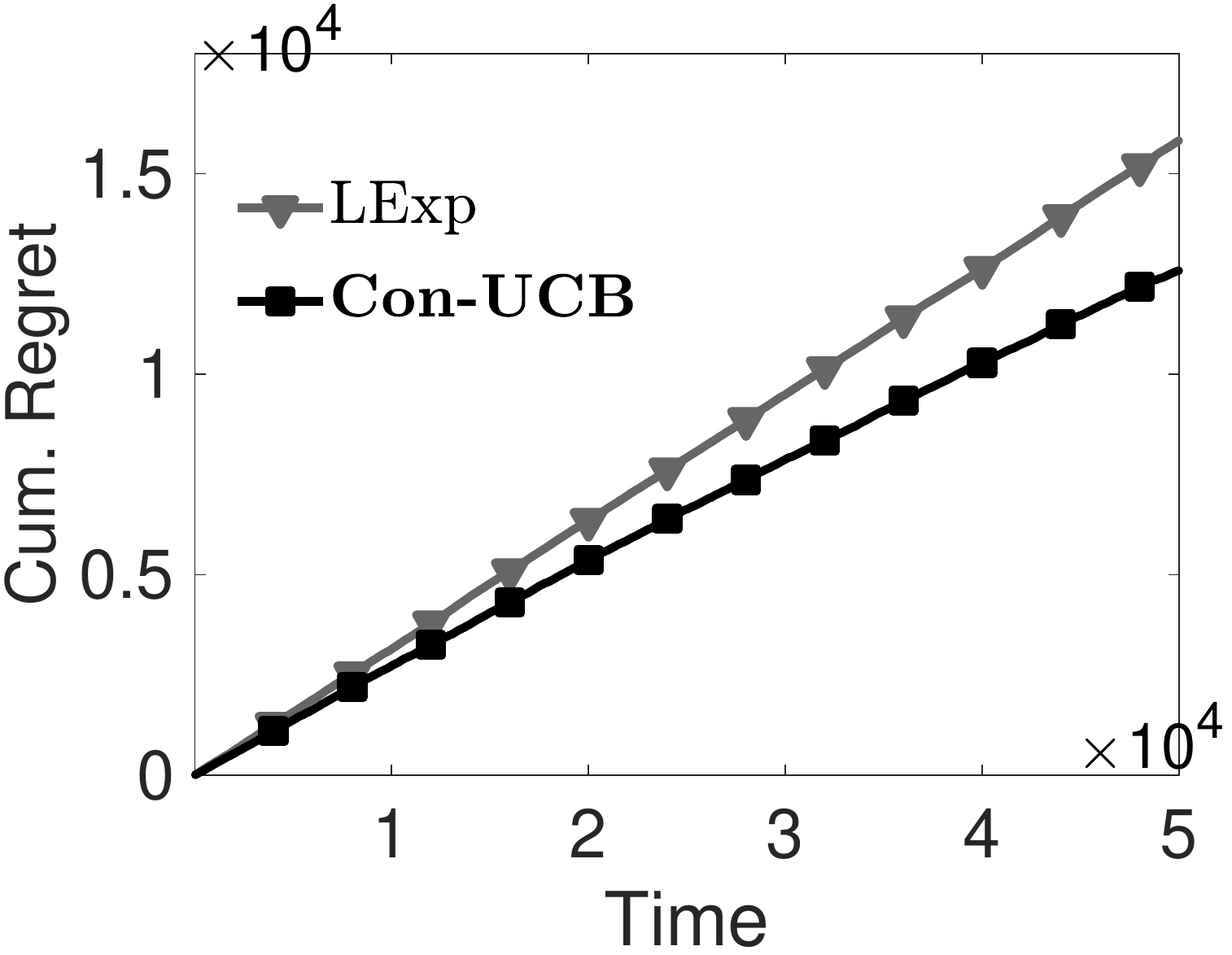}
 \label{fig:edx:cumregret}
  }
  \subfigure[Cumulative violation]{
  \includegraphics[width=0.232\textwidth]{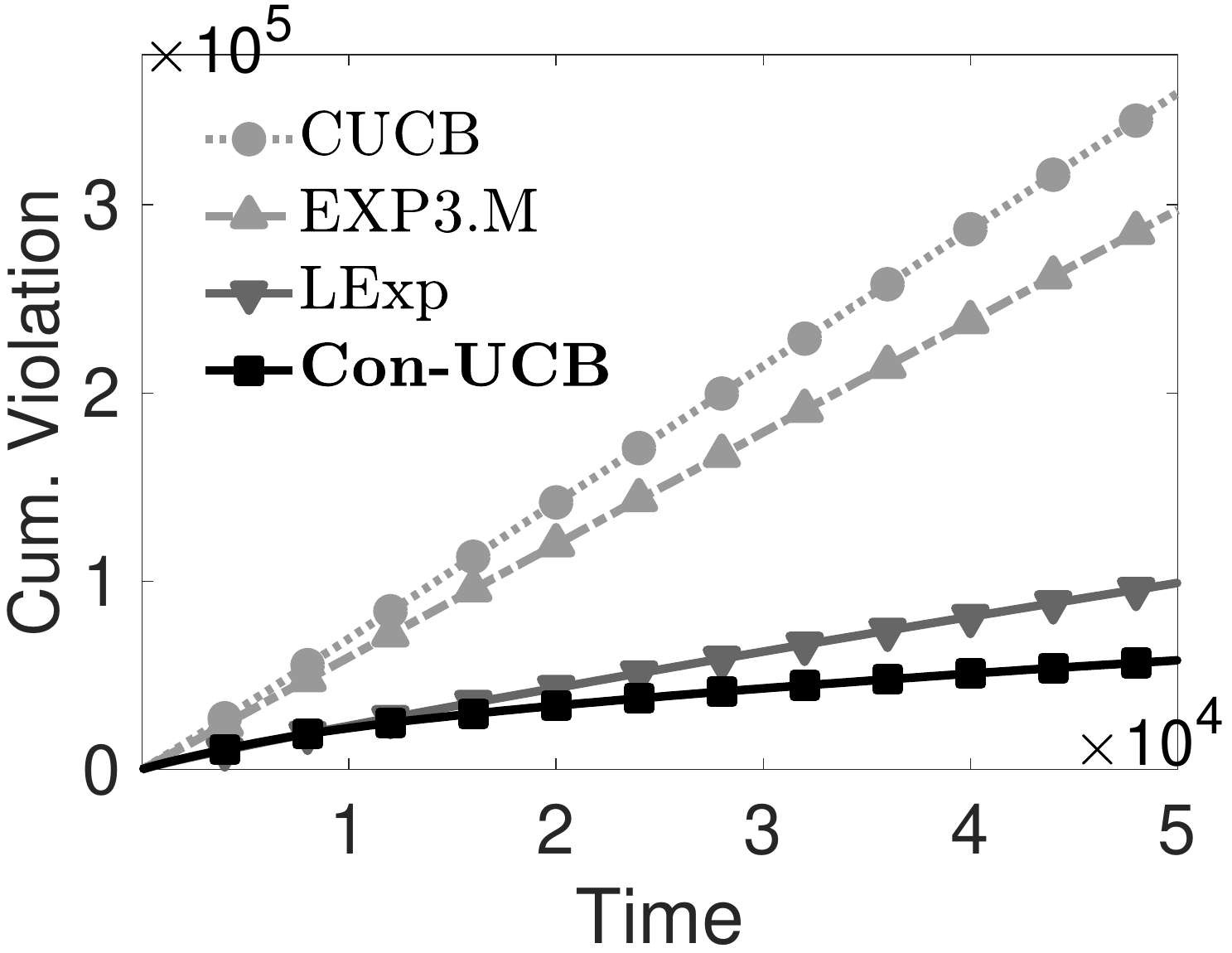}
 \label{fig:edx:cumviolation}
  }
  \subfigure[Cum. compound reward]{
  \includegraphics[width=0.232\textwidth]{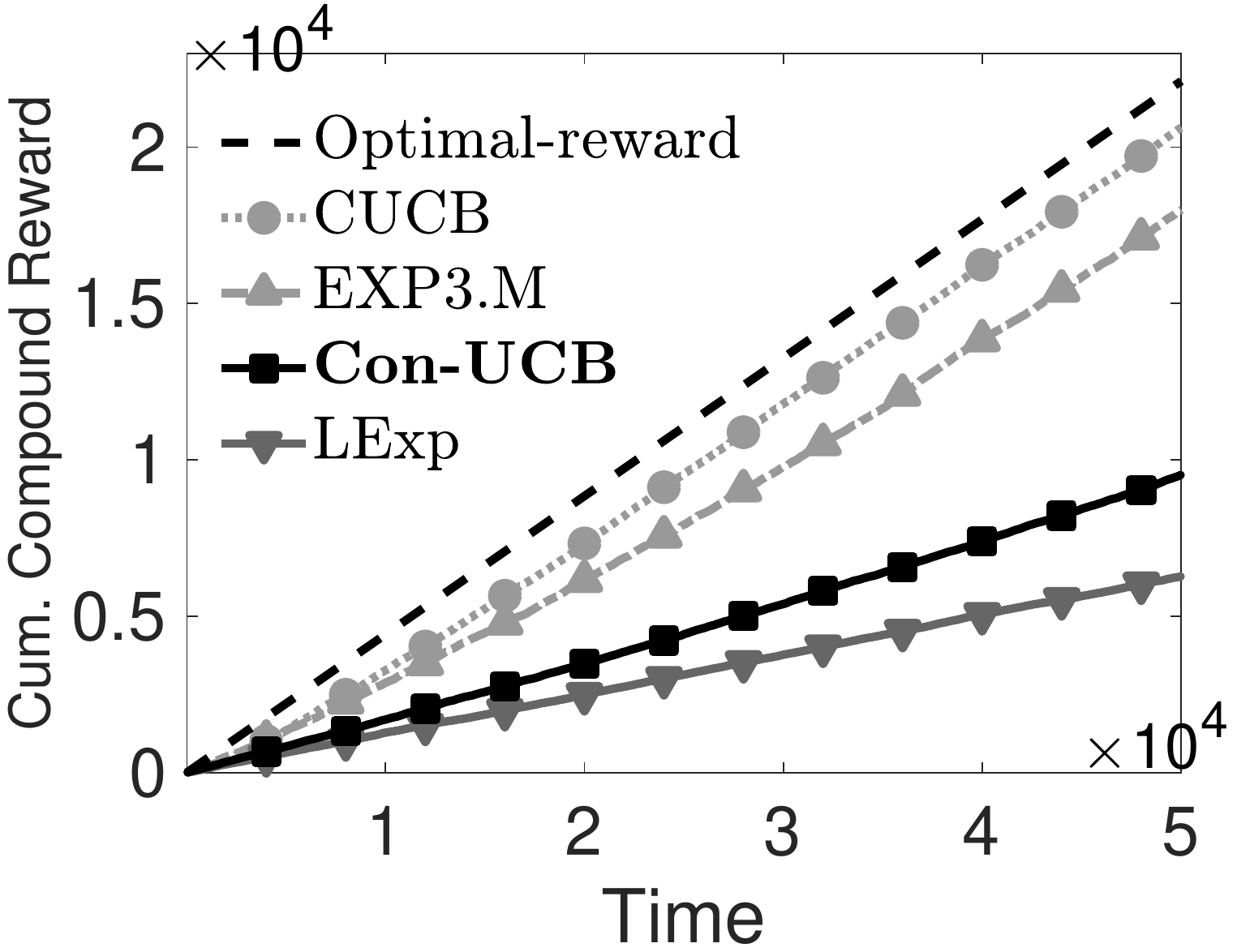}
 \label{fig:edx:cumreward}
  }
 \subfigure[Reward/Violation ratio]{
 \includegraphics[width=0.232\textwidth]{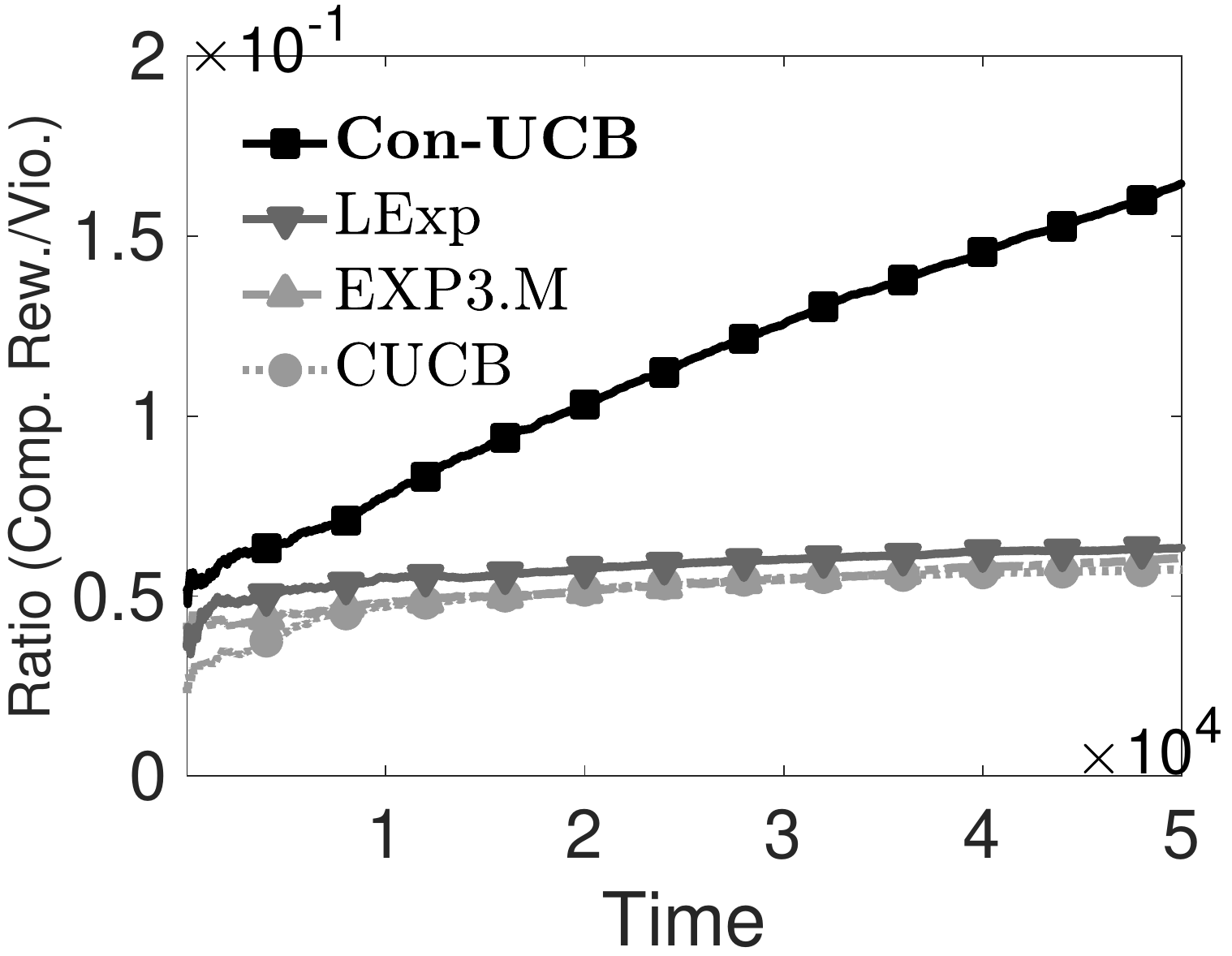}
 \label{fig:edx:ratio}
  }
\vspace{-0.1in}
\caption{Experiment results on the \edxcourse{} dataset. $K=290, L=60, h=10, \delta=0.05$.}
 \label{fig:edx}
\end{figure*}

\section{Experiments}
\label{sec:experiments}
We conduct experiments on three real-world datasets to evaluate the performance
of \myucb{}. 
Two datasets, \coupon{}~\cite{kaggle2016coupon} and
\adclick{}~\cite{kaggle2015avito}, with $271$ coupons and $225$ ads
respectively, are shown to have a two-level feedback structure
in~\cite{cai2017}. 
In particular, for each coupon in \coupon{}, a user who clicks the link
to the coupon can decide whether to purchase that coupon; for each ad in
\adclick{}, a user who clicks the link to the ad can decide whether to request the 
corresponding seller's phone number.
Thus, for \coupon{} (\adclick{}), the first-level feedback is the CTR of
each coupon (the CTR of each ad) and the second-level feedback is the purchase
rate of each coupon (the phone request rate of each ad). 
The third dataset, \edxcourse{}, is extracted from the data on $290$ Harvard and
MIT edX online courses~\cite{chuang2016harvardx}. 
In particular, for the $290$ online courses, we obtain course participation
rates by normalizing the numbers of participants using min-max scaling and treat
the course participation rates as the first-level feedback; we calculate course
certification rates by dividing the numbers of certified participants by the
numbers of participants, and treat the course certification rates as the second-level feedback. 

We treat the coupons, ads, and courses as different sets of arms. 
To simulate the real-time two-level feedback of the coupons, ads, and courses,
we generate the first-level reward of each arm (coupon, ad, and course) using a
Bernoulli variable with mean taken from the first-level feedback (coupon CTR, ad
CTR, and course participation rate) in the three datasets, and generate the
second-level reward of each arm using another independent Bernoulli variable with
mean taken from the second-level feedback (coupon purchase rate, ad phone request
rate, and course certification rate). 
 
For comparison purposes, we implement three state-of-the-art bandit algorithms
that can select multiple arms $(L\ge 1)$ at each round as
baselines, i.e., 
\textsc{CUCB}~\cite{chen2013combinatorial},
\textsc{Exp3.M}~\cite{uchiya2010algorithms} and \lexp{}~\cite{cai2017}. 
Specifically, \textsc{CUCB} selects the top-$L$ arms with the $L$ highest UCB
indices $\bar{g}_{i}^{t} + \sqrt{3\ln t/(2 N_{i}^{t})}$. 
\textsc{Exp3.M} selects $L$ arms using exponential weights on the compound
rewards of $K$ arms, and \lexp{} selects arms using exponential weights based on
the Lagrangian function of reward and violation of $K$ arms.

For the three datasets, we run the three algorithms together with \myucb{} for
$50,000$ rounds with parameter settings as shown in
Figure~\ref{fig:coupon}--\ref{fig:edx}, respectively. 
In particular, the parameters of \textsc{Exp3.M} and \lexp{} are set in accordance
with Corollary 1 of \cite{uchiya2010algorithms} and Theorem 1 of
\cite{cai2017}, respectively.
We compare the cumulative regrets $t{\bm{x}^{*}}^{\intercal}\bm{g}-\sum_{\tau=1}^t\sum_{i\in
  \mathcal{I}_{\tau}}g_i^{\tau}$ of \lexp{} and \myucb{} at each round $t$,
where the optimal policy $\bm{x}^{*}$ is computed from the means of the
two-level feedback taken from each datatset.
(Note that the regrets of \textsc{CUCB} and \textsc{Exp3.M} are not considered
since they both have an unconstrained optimal policy, and therefore have
different regret definitions from \lexp{} and \myucb{}.)
We also compare the cumulative violations $\sum_{\tau=1}^t(h-\sum_{i\in
  \mathcal{I}_{\tau}}a_i^{\tau})_+$ and the cumulative compound rewards
$\sum_{\tau=1}^t\sum_{i\in \mathcal{I}_{\tau}}g_i^{\tau}$ of the four
algorithms. 
To put things into perspective, we compare the ratios between the cumulative
rewards and the cumulative violations of all the algorithms. 
Such ratios show how much reward an algorithm can gain for each unit violation
it has made.

The experiment results are averaged over $200$ runs of each algorithm and
illustrated in Figure~\ref{fig:coupon}--\ref{fig:edx}. 
Figure~\ref{fig:coupon:cumregret} shows that the cumulative regret of \myucb{}
is much lower than that of \lexp{} on the \coupon{} dataset. 
This shows that \myucb{} can reduce the regret significantly by selecting arms
using UCB-based optimization instead of exponential weights as in \lexp{}.
Figure~\ref{fig:coupon:cumviolation} and Figure~\ref{fig:coupon:cumreward} show
the cumulative violations and the cumulative rewards of the four algorithms.
In particular, the \emph{Optimal-reward} in Figure~\ref{fig:coupon:cumreward}
shows the cumulative reward $t{\bm{x}^{*}}^{\intercal}\bm{g}$ of the optimal policy $\bm{x}^{*}$
at each round $t$. 
As shown in Figure~\ref{fig:coupon:cumreward}, \ucb{} and \expm{} have larger
cumulative rewards than \myucb{} and \lexp{}, as both \ucb{} and \expm{} neglect
the threshold constraint and thereby blindly selecting arms that maximize the
cumulative rewards. 
Therefore, both \ucb{} and \expm{} incur huge cumulative violations as shown in
Figure~\ref{fig:coupon:cumviolation}. 
Moreover,
\myucb{} has a larger cumulative reward and a lower cumulative violation than
\lexp{}. 
This matches our theoretical results that \myucb{} has smaller regret as well as
violation bounds than \lexp{}.
Figure~\ref{fig:coupon:ratio} shows that \myucb{} achieves the largest
reward/violation ratios among the four algorithms. 
This means that \myucb{} achieves the best tradeoff between rewards and violations
and accumulates most reward for each unit violation it incurs.

We have similar experiment results on \adclick{} and \edxcourse{} to those on
\coupon{}. 
As shown in Figure~\ref{fig:ad} and Figure~\ref{fig:edx}, \myucb{} achieves
lower cumulative regret and higher cumulative rewards than \lexp{}, and has the
lowest cumulative violations and largest reward/violation ratios among all
algorithms. 
Due to space limit, we omit the \blue{details}.

In summary, our experiment results are consistent with our theoretical analysis
and demonstrate the effectiveness of our \myucb{} algorithm in selecting arms
with high cumulative rewards as well as low cumulative violations, thus achieving 
a good tradeoff between the reward and the violation.

%% file: sec6-conclusion.tex
\section{Conclusion}
\label{sec:conclusion}
\noindent In this paper, we consider the web link selection problem
with multi-level feedback. 
We formulate it as a constrained multiple-play stochastic multi-armed bandit
problem with multi-level reward. 
We design an efficient algorithm \myucb{} for solving the problem, and prove
that for any given allowed failure probability $\delta\in(0,1)$, with
probability at least $1-\delta$, \myucb{} guarantees $O(\sqrt{T\ln
  \frac{T}{\delta}})$ regret and violation bounds. 
We conduct extensive experiments on three real-world datasets to compare our
\myucb{} algorithm with state-of-the-art context-free bandit algorithms. 
Experiment results show that  \myucb{}  balances regret and
violation better than the other algorithms and outperforms \lexp{} in both
regret and violation.

%% file: appendices.tex
\appendix


\section{Proof of Lemma \ref{lem:reward-bound}} \label{apd:proof:reward-bound}
\begin{proof}
  We first show that \eqref{eq:a-upper} and \eqref{eq:a-concentration} hold with
  probability at least $1-\frac{\delta}{2}$. Notice that
  $\gamma=72\ln\frac{8KT}{\delta}\ge1$. From Lemma \ref{lem:ucb}, by taking a
  union bound over all $i\in\mathcal{K}$ and all $t$, we obtain that for all $1\le
  i\le K$ and all $1\le t\le T$, with probability at least
  $1-2KTe^{-\frac{1}{72}\gamma}=1-\frac{\delta}{4}$,
\begin{equation}
|\bar{a}_{i}^{t}-a_{i}|\le 2R(\bar{a}_{i}^{t},N_{i}^{t}+1), \label{eq:ucb}
\end{equation}
which means 
\begin{equation}
a_{i}\le\bar{a}_{i}^{t}+2R(\bar{a}_{i}^{t},N_{i}^{t}+1). \label{eq:ucb-right}
\end{equation}
 Recall that $\hat{a}_{i}^{t}=\min\{1,\bar{a}_{i}^{t}+2R(\bar{a}_{i}^{t},N_{i}^{t}+1)\}$. 
 Together with \eqref{eq:ucb-right}, we see that \eqref{eq:a-upper} holds.

To prove \eqref{eq:a-concentration}, we define a series of random variables
$Z_{t},1\le t \le T$ as
\begin{equation*}
Z_{t}=\sum_{i\in\mathcal{I}_{t}}a_{i}^{t}-\sum_{i\in\mathcal{I}_{t}}a_{i}.
\end{equation*}
We know $\Ex\{Z_{t}|H_{t-1}\}=0$ and $|Z_{t}|\le L$. Recall that $H_{t}$ denotes
the historical information of chosen actions and observations up to time $t$.
Thus, by Lemma \ref{lem:azuma}, we get, with probability at least
$1-\frac{\delta}{8}$,
\begin{equation}
|\sum_{t=1}^{T}(\sum_{i\in\mathcal{I}_{t}}a_{i}^{t}-\sum_{i\in\mathcal{I}_{t}}a_{i})|\le L\sqrt{2T\ln\frac{16}{\delta}}. \label{eq:a-divergence-rv}
\end{equation}
Similarly, with probability at least $1-\frac{\delta}{8}$,
\begin{equation}
|\sum_{t=1}^{T}(\sum_{i\in\mathcal{I}_{t}}\hat{a}_{i}^{t}-\bm{x}_{t}^{\intercal}\hat{\bm{a}}^{t})|\le L\sqrt{2T\ln\frac{16}{\delta}}. \label{eq:a-divergence-rs}
\end{equation}
Next we bound $|\sum_{t=1}^{T}\sum_{i\in\mathcal{I}_{t}}(\hat{a}_{i}^{t}-a_{i})|$. Notice that \eqref{eq:ucb} also implies that for all $i$ and $t$, 
\begin{equation*}
|\hat{a}_{i}^{t}-a_{i}|\le 4R(\bar{a}_{i}^{t},N_{i}^{t}+1).
\end{equation*}
Let $\tau(i,n)$ denote the time that arm $i$ is played for the $n$th time. We have
\begin{align}
&|\sum_{t=1}^{T}\sum_{i\in\mathcal{I}_{t}}(\hat{a}_{i}^{t}-a_{i})| \le \sum_{t=1}^{T}\sum_{i\in\mathcal{I}_{t}}4R(\bar{a}_{i}^{t},N_{i}^{t}+1) \nonumber\\
&= \sum_{i=1}^{K}\sum_{n=1}^{N_{i}^{T+1}}4R(\bar{a}_{i}^{\tau(i,n)},n) \nonumber\\
&\le\sum_{i=1}^{K}\sum_{n=1}^{N_{i}^{T+1}}4(\sqrt{\frac{\gamma}{n}}+\frac{\gamma}{n}) \nonumber\\
&=O(\sum_{i=1}^{K}(\sqrt{\gamma N_{i}^{T+1}}+\gamma\ln N_{i}^{T+1})) \nonumber\\
&\le O(\sqrt{K}\sqrt{\sum_{i=1}^{K}\gamma N_{i}^{T+1}}+K\gamma\ln T) \label{eq:proof:a-divergence-ucb}\\
&=O(L\sqrt{KT\ln\frac{KT}{\delta}}), \label{eq:a-divergence-ucb}
\end{align}
where \eqref{eq:proof:a-divergence-ucb} follows from the Cauchy-Schwarz
inequality and \eqref{eq:a-divergence-ucb} follows from the fact that
$\sum_{i=1}^{K}N_{i}^{T+1}=LT$. Thus, \eqref{eq:a-divergence-rv},
\eqref{eq:a-divergence-rs} and \eqref{eq:a-divergence-ucb} together give
\begin{equation*}
|\sum_{t=1}^{T}(\sum_{i\in\mathcal{I}_{t}}a_{i}^{t}-\bm{x}_{t}^{\intercal}\hat{\bm{a}}^{t})|=O(L\sqrt{KT\ln\frac{KT}{\delta}}).
\end{equation*}

Repeating the same analysis, we can show that \eqref{eq:g-upper} and
  \eqref{eq:g-concentration} also hold with probability at least
  $1-\frac{\delta}{2}$.
Then, we can prove the lemma using the union bound.
\end{proof}
